\documentclass[10pt,twocolumn,letterpaper]{article}

\usepackage[ruled,vlined,linesnumbered]{algorithm2e}
\usepackage[noend]{algpseudocode}
\usepackage{enumerate}
\usepackage{comment}
\usepackage{times}
\usepackage{epsfig}
\usepackage{graphicx}
\usepackage{amsmath}
\usepackage{amssymb}
\usepackage{xcolor}
\usepackage{booktabs}
\usepackage{multirow}
\definecolor{darkgreen}{rgb}{0.1,0.5,0.1}
\usepackage[english]{babel}
\newtheorem{theorem}{Theorem}

\newtheorem{lemma}[theorem]{Lemma}
\newtheorem{proposition}{Proposition}
\newtheorem{proof}{Proof}
\usepackage[pagenumbers]{iccv} 

\definecolor{iccvblue}{rgb}{0.21,0.49,0.74}
\usepackage[pagebackref,breaklinks,colorlinks,allcolors=iccvblue]{hyperref}

\def\paperID{13536} 
\def\confName{ICCV}
\def\confYear{2025}

\title{Marginalized Generalized IoU (MGIoU): \\ A Unified Objective Function for Optimizing Any Convex Parametric Shapes}

\author{Duy-Tho Le$^{1}$, Trung Pham$^{2}$, Jianfei Cai$^{1}$, Hamid Rezatofighi$^{1}$\\
$^{1}$Monash University, $^{2}$NVIDIA \\
{\tt\small \{tho.le1, jianfei.cai, hamid.rezatofighi\}@monash.edu, trung.pham@nvidia.com}
}

\begin{document}
\maketitle

\begin{abstract}
\label{sec:introduction}
Optimizing the similarity between parametric shapes is crucial for numerous computer vision tasks, where Intersection over Union (IoU) stands as the canonical measure. However, existing optimization methods exhibit significant shortcomings: regression-based losses like L1/L2 lack correlation with IoU, IoU-based losses are unstable and limited to simple shapes, and task-specific methods are computationally intensive and not generalizable across domains. As a result, the current landscape of parametric shape objective functions has become scattered, with each domain proposing distinct IoU approximations. To address this, we unify the parametric shape optimization objective functions by introducing \textbf{Marginalized Generalized IoU (MGIoU)}, a novel loss function that overcomes these challenges by projecting structured convex shapes onto their unique shape Normals to compute one-dimensional normalized GIoU. MGIoU offers a simple, efficient, fully differentiable approximation strongly correlated with IoU. We then extend MGIoU to \textbf{MGIoU$^+$} that supports optimizing unstructured convex shapes. Together, MGIoU and MGIoU$^+$ unify parametric shape optimization across diverse applications. Experiments on standard benchmarks demonstrate that MGIoU and MGIoU$^+$ consistently outperform existing losses while reducing loss computation latency by 10-40x. Additionally, MGIoU and MGIoU$^+$ satisfy metric properties and scale-invariance, ensuring robustness as an objective function. We further propose \textbf{MGIoU$^-$} for minimizing overlaps in tasks like collision-free trajectory prediction. Code is available at https://ldtho.github.io/MGIoU/
\end{abstract}

\section{Introduction}
Parametric shape optimization is a core problem in computer vision, robotics, and graphics, with applications spanning 2D/3D object detection \cite{yolov1, kld,gwd, bi2021disentangled}, 6D pose estimation \cite{omni3d, unimode}, shape registration, and trajectory prediction \cite{waymo, mtr, mtr++}. At its heart, this optimization seeks to maximize the similarity between predicted and ground-truth shapes, commonly quantified using the Intersection over Union (IoU). IoU measures the overlap between two parametric shapes relative to their combined area or volume, providing a standardized and scale-invariant metric widely adopted for evaluating shape similarity.

Despite its popularity, directly optimizing IoU (or its improved variant, Generalized IoU (GIoU)~\cite{giou}) for arbitrary convex shapes presents significant challenges. Computing the IoU/GIoU analytically for many convex shapes, \eg two 3D ellipsoids, is non-trivial. Even for simpler shapes like rotated rectangles, 2D polygons or 3D cuboids~\cite{pytorch3d}, analytical calculation is feasible, but using IoU/GIoU directly as an objective function within gradient-based optimization remains challenging and computationally expensive. Moreover, certain shape parameters, such as rotation angles, lead to erratic gradient landscapes, making optimization unstable when using IoU/GIoU losses. Consequently, aside from a few cases — such as axis-aligned bounding boxes~\cite{giou} or 3D bounding boxes with 1 rotation DoF~\cite{mmcv, zhou2019iou} — IoU/GIoU-based optimization has not been widely adopted for general shape optimization tasks.

As a result, the objective function for parametric shape optimization is highly fragmented in different application domains. Different applications either: \emph{(i)} rely on basic losses shown to work empirically, such as simple parametric regression losses (e.g., L-norms), OKS~\cite{oks}, or Chamfer distance for shape vertices~\cite{qi2017pointnet}, or \emph{(ii)} adopt approximations of IoU/GIoU, such as Gaussian-based models for rotated boxes~\cite{yang2021r3det, gwd, kfiou}, or vertex-based approaches for quadrilaterals~\cite{bi2021disentangled, qrn, dmp, textboxes++}. However, these approaches often lack direct correlation with the true IoU/GIoU and in some cases, they may violate important properties, such as scale invariance or metric consistency. Moreover, many of these methods require extensive task-specific tuning~\cite{gwd,kld}, ultimately leading to suboptimal shape alignment or over-fitted to specific datasets. Thus, there is currently no unified objective function that works effectively across different applications and shape parameterizations while maintaining strong correlation with standard IoU-based metrics.

In this work, we introduce \textbf{Marginalized Generalized IoU (MGIoU)}, a novel geometric loss function explicitly designed to enable flexible and stable optimization of arbitrary convex parametric shapes in any dimension. MGIoU simplifies complex shape overlap calculations into efficient marginalization of 1D GIoU operations by projecting the shapes onto their associated normals (\eg, edge normals in 2D, face normals in 3D, or semi-axes for ellipses/ellipsoids). These projections allow for a differentiable, robust, and stable overlap evaluation, even when the shapes do not physically overlap in space.

MGIoU unifies shape optimization into a single coherent loss term applied directly to shape normals, enabling holistic adjustment of shape vertices and parameters (\eg, position, size, and orientation) without requiring balancing or fine-tuning of multiple separate loss terms. Its simplicity allows seamless integration into existing optimization pipelines, making it a drop-in replacement across diverse applications such as 2D oriented object detection, 3D shape estimation, and polygonal shape fitting. Furthermore, the mathematical construction of MGIoU ensures compliance with essential scale-invariance and metric properties, reinforcing its theoretical soundness as a reliable measure.

We propose three variations of this objective function, making MGIoU applicable across a wide range of tasks. The primary version, MGIoU, is designed for optimizing \emph{structured} convex shapes, where both the source and target shapes share the same parametric domain, \eg both being rectangles, ellipses, cuboids \etc (see \cref{fig:teaser} (C, D)).

An extended version, denoted as \textbf{MGIoU$^+$}, supports optimization between \emph{unstructured} convex shapes, where the source and target may belong to different shape families or have differing numbers of vertices, such as convex polygons or polyhedra with varying and arbitrary vertex and edge counts (see \cref{fig:teaser} (A, B) as two examples). To accommodate this, we use a vertex-based parametrization for both shapes, under the condition that the target shape always has more vertices than the source. In MGIoU$^+$, we introduce an additional convexity regularizer to ensure predicted shapes remains convex during optimization.

We also propose a complementary loss, denoted as \textbf{MGIoU$^-$}, which is specifically designed for minimizing overlap between two arbitrary convex shapes. This objective is particularly useful for applications such as agent trajectory prediction, where the goal is to reduce the risk of collisions between predicted trajectories. Unlike standard shape alignment objectives, MGIoU$^-$ directly minimizes the IoU between the predicted and reference shapes. By applying the same shape projection mechanism, it robustly penalizes overlaps, encouraging safer trajectory predictions. We further extend MGIoU$^-$ to handle sequences of shapes, enabling \emph{spatio-temporal} optimization for tracking purposes.

We evaluate the effectiveness of MGIoU and its variants, MGIoU$^+$ and MGIoU$^-$, across multiple applications on standard benchmarks. Our results consistently show improved performance over traditional losses, while also significantly reducing the computational overhead compared to more complex existing losses.
\\
Our key contributions can be summarized as follows:
\begin{itemize}
    \item We propose \textbf{MGIoU}, a novel geometric loss function for the flexible and stable optimization of arbitrary convex parametric shapes in any dimension. MGIoU simplifies complex shape overlap calculations into efficient marginalization of 1D GIoU operations by projecting the shapes onto their associated normals. It unifies position, dimension, and orientation optimization into a single, differentiable objective, eliminating the need for task-specific loss balancing or tuning.
    \item We extend MGIoU to \textbf{MGIoU$^+$}, which handles optimization between \emph{unstructured} convex shapes with different parametric structures, and \textbf{MGIoU$^-$}, a complementary loss designed for minimizing shape overlap in applications such as collision-free trajectory prediction.
    \item We conduct extensive experiments across diverse tasks, including 2D/3D shape alignment, oriented object detection, and trajectory prediction. Our results demonstrate the effectiveness of MGIoU and its variants in improving performance while significantly reducing computational overhead compared to existing shape optimization losses.
\end{itemize}

\section{Related Works}
\label{sec:related_work}

The selection of loss functions in shape optimization has evolved significantly to address the diverse demands of fundamental vision and robotics tasks, including various object detection tasks (\eg, 2D, 3D, and quadrilateral object detection), 6-DoF object pose estimation, state regression in trajectory forecasting \etc. This section reviews the development of loss functions across these domains, highlighting their limitations and motivating the need for a unified approach like Marginalized Generalized IoU (MGIoU).

\noindent \textbf{Loss Functions in 2D Object Detection.} 
Loss function design in object detection has progressed to optimize localization accuracy and alignment with Intersection over Union (IoU), the standard evaluation metric. Early models~\cite{yolov1,yolov2,fasterrcnn} relied on L1/L2 regression losses to minimize coordinate differences between predicted and ground-truth bounding boxes. While computationally efficient, these losses lack direct correlation with IoU, often leading to suboptimal overlap~\cite{giou}. To address this, IoU-based losses emerged for 2D axis-aligned detection, with Generalized IoU (GIoU)~\cite{giou} introducing a differentiable penalty based on the tightest enclosing box to handle non-overlapping cases, though it does not extend beyond 2D axis-aligned geometries. Subsequent alternatives like Distance IoU (DIoU)~\cite{diou} and Complete IoU (CIoU)~\cite{ciou} enhance convergence by incorporating distance and aspect ratio terms but remain limited to axis-aligned boxes and require careful tuning.

For 2D oriented object detection, which is critical in aerial imagery and text recognition, losses like modulated rotation loss~\cite{qian2021learning} mitigate angle periodicity issues, while KFIoU~\cite{kfiou}, Gaussian Wasserstein Distance (GWD)~\cite{gwd}, and Kullback-Leibler Divergence (KLD)~\cite{kld} model bounding boxes as Gaussian distributions to approximate IoU. Skew IoU (SIoU)~\cite{yang2021r3det} further targets skewed geometries. However, these losses are computationally intensive and tailored specifically for rotated bounding boxes.

In quadrilateral detection, which is widely used in document analysis, methods such as Quadbox~\cite{keserwani2021quadbox} employ vertex regression with L1 losses, while Textboxes++~\cite{textboxes++} resolves vertex ordering via rectangular distances. QRN~\cite{qrn} sorts vertices by polar angles. Despite their effectiveness in specific tasks, these solutions are often complex, task-specific, and lack a direct correlation with IoU, necessitating further refinement.

\noindent\textbf{Loss Functions in 3D Object Recognition \& 6-DoF Object Pose Estimation.}  In 3D object Regconition, crucial for applications like autonomous driving, bounding boxes are defined by position, dimensions, and orientation, requiring loss functions that account for both spatial and rotational components. Common 3D object recognition approaches assume only 1-DoF in rotation (usually yaw), recently the attention shifts to 3-DoF (yaw, pitch, roll) due to the need for Virtual Reality \cite{omni3d, objectron, hypersim, arkitscenes} and 6-DoF Object Pose Estimation \cite{shapenet, scannet}. Many approaches combine regression losses for box parameters with adaptations for 3D geometry. For example, SO3 losses~\cite{su2020deep} operate directly on the rotation manifold to ensure proper handling of 3D orientations, while Chamfer distance~\cite{qi2017pointnet} measures bidirectional point set distances, commonly used in shape registration and applicable to 3D detection. However, both SO3 and Chamfer distance lack direct correlation with IoU, limiting their effectiveness in optimizing bounding-box orientation and overlap in 3D space. This underscores the need for a unified loss function that jointly optimizes spatial and rotational alignment while maintaining IoU consistency.

\noindent\textbf{Collision-aware Penalties Losses.} Collision-aware losses incorporate collision avoidance and feasibility penalties directly into the training objective, ensuring that predicted trajectories not only align closely with ground truth but also conform to essential physical and social driving constraints. For instance, \cite{niedoba2019improving} proposes an auxiliary off-road loss to penalize off-road trajectories. In \cite{kim2022diverse}, an auxiliary LaneLoss is introduced to encourage diverse trajectories enough to cover every possible maneuver. Similarly, RouteLoss \cite{zhang2022trajgen} encourages generated trajectories to remain close to the nearest available route, thereby reducing off-road predictions without sacrificing diversity. However, these losses primarily address interactions between static objects and the ego vehicle, neglecting dynamic interactions among multiple road agents. Closest in spirit to our proposed MGIoU$^-$, TrafficSim \cite{trafficsim} introduces a common-sense loss aimed at avoiding collision trajectories by approximating objects using five circles and computing the L2 distance between centroids. While this approach captures the overall proximity of objects, it does not precisely reflect the true boundaries of road agents and is tailored for simulation training, limiting its direct applicability to real-world dynamic interactions.

\noindent\textbf{Summary and Motivation for MGIoU.} The evolution of loss functions across various tasks, including object detection, shape registration, and trajectory prediction, highlights a trade-off between accuracy, computational efficiency, and generalizability. In object detection, standard regression losses lack alignment with IoU, while IoU-based losses like GIoU struggle with non-axis-aligned boxes. Task-specific losses for rotated, 3D, and quadrangle detection offer improvements but are often computationally expensive and narrowly focused. In shape registration, traditional distance metrics like Hausdorff or Chamfer may not directly correlate with overlap, and in trajectory prediction, distance-based penalties for collision avoidance may not capture spatial overlaps effectively. MGIoU addresses these shortcomings by providing a unified, efficient loss function that leverages 1D projection to optimize spatial relationships across these diverse tasks.
\vspace{-0.5em}
\section{Methodology}
\vspace{-0.5em}
\label{sec:method}
\begin{algorithm}[!t]
\caption{\textbf{MGIoU and MGIoU$^+$} - Maximize Shape Overlap}
\label{algo:MGIoU}
\small{
\SetKwInOut{Input}{input}\SetKwInOut{Output}{output}
\Input{Predicted vertices \(\mathbf{P} \in \mathbb{R}^{N_P \times D}\), ground truth vertices \(\mathbf{G} \in \mathbb{R}^{N_G \times D}\), convexity weight \(\lambda\), shape type (structured or unstructured)}
\Output{\(\mathcal{L}_{\text{total}}\)}
Compute unique normals $\mathcal{A}$ of $\mathbf{P}$ and $\mathbf{G}$ \\
Initialize overlap sum $S = 0$ \\
\For{each unique normal $\mathbf{a}_i \in \mathcal{A}$}{
  $\text{proj}_{\mathbf{P}, \mathbf{a}_i} = \mathbf{P} \cdot \mathbf{a}_i$, $\text{proj}_{\mathbf{G}, \mathbf{a}_i} = \mathbf{G} \cdot \mathbf{a}_i$ \\
  $\min_{\mathbf{P}, \mathbf{a}_i} = \min(\text{proj}_{\mathbf{P}, \mathbf{a}_i})$, $\max_{\mathbf{P}, \mathbf{a}_i} = \max(\text{proj}_{\mathbf{P}, \mathbf{a}_i})$ \\
  $\min_{\mathbf{G}, \mathbf{a}_i} = \min(\text{proj}_{\mathbf{G}, \mathbf{a}_i})$, $\max_{\mathbf{G}, \mathbf{a}_i} = \max(\text{proj}_{\mathbf{G}, \mathbf{a}_i})$ \\
  $|P_{A_i} \cap G_{A_i}| = \max\left(0, \min(\max_{\mathbf{P}, \mathbf{a}_i}, \max_{\mathbf{G}, \mathbf{a}_i}) \right. \left. - \max(\min_{\mathbf{P}, \mathbf{a}_i}, \min_{\mathbf{G}, \mathbf{a}_i})\right)$ \\
  $|P| = \max_{\mathbf{P}, \mathbf{a}_i} - \min_{\mathbf{P}, \mathbf{a}_i}$ \\
  $|G| = \max_{\mathbf{G}, \mathbf{a}_i} - \min_{\mathbf{G}, \mathbf{a}_i}$ \\
  $|P_{a_i} \cup G_{a_i}| = |P| + |G| - |P_{a_i} \cap G_{a_i}|$ \\
  $|C| = \max(\max_{\mathbf{P}, \mathbf{a}_i}, \max_{\mathbf{G}, \mathbf{a}_i}) - \min(\min_{\mathbf{P}, \mathbf{a}_i}, \min_{\mathbf{G}, \mathbf{a}_i})$ \\
  $\text{IoU} = \frac{|P \cap G|}{|P \cup G|}$ \\
  $\text{GIoU}^{1D}_i = \text{IoU} - \frac{|C| - |P_{a_i} \cup G_{a_i}|}{|C|}$ \\
  $S = S + \text{GIoU}^{1D}_i$
}

$\text{MGIoU} = \frac{S}{|\mathcal{A}|}$  

$\mathcal{L}_{\text{MGIoU}} = \frac{1 - \text{MGIoU}}{2}$ \\

\If{shape is unstructured}{
  Compute $\mathcal{L}_{\text{convexity}}$: \\
  \For{each $i = 0$ to $N_P - 1$}{
    $\mathbf{e}_i = \mathbf{p}_{(i+1) \mod N_P} - \mathbf{p}_i$ \\
    $\mathbf{n}_i = (-e_{i,y}, e_{i,x})$ \\
    $s_1 = \sum_{j=0}^{N_P-1} \max(0, -(\mathbf{p}_j - \mathbf{p}_i) \cdot \mathbf{n}_i)$ \\
    $s_2 = \sum_{j=0}^{N_P-1} \max(0, (\mathbf{p}_j - \mathbf{p}_i) \cdot \mathbf{n}_i)$ \\
    $\text{penalty}_i = \min(s_1, s_2)$ \\
  }
  $\mathcal{L}_{\text{convexity}} = \frac{1}{N_P} \sum_{i=0}^{N_P-1} \text{penalty}_i$
}
\Else{
  \Return{$\mathcal{L}_{\text{MGIoU}}$}
}
$\mathcal{L}_{\text{MGIoU}^+} = \mathcal{L}_{\text{MGIoU}} + \lambda \mathcal{L}_{\text{convexity}}$ \\
\Return{$\mathcal{L}_{\text{MGIoU}^+}$}
}
\end{algorithm}

We introduce \textbf{Marginalized Generalized IoU (MGIoU)}, a novel loss function designed to optimize shape alignment for \textit{structured} convex shapes in object detection tasks, along with its variants: \textbf{MGIoU\(^+\)} for \textit{unstructured} convex shapes and \textbf{MGIoU\(^-\)} for minimizing overlaps. The key idea of MGIoU is to project vertices of arbitrary convex shapes onto their normals, computing a GIoU in 1D projection space. This enables MGIoU to generalize across 2D rotated and 3D 6-DoF detection tasks, unifying position, size, and orientation optimization in a single, differentiable loss term. MGIoU\(^+\) extends this to unstructured shapes with a convexity regularizer, quantified on quadrilateral object detection task, while MGIoU\(^-\) adapts it to penalize overlaps, enhancing collision avoidance in trajectory prediction.
\vspace{-0.5em}
\subsection{Identifying Shape Normals}
\vspace{-0.5em}
Given two arbitrary shapes \(P, G \subseteq \mathbb{R}^D\) (\(D=2\) for 2D, \(D=3\) for 3D), with vertices \(\mathbf{P} \in \mathbb{R}^{N_P \times D}\) and \(\mathbf{G} \in \mathbb{R}^{N_G \times D}\), where \(N_P\) and \(N_G\) are the numbers of vertices for \(P\) and \(G\) respectively, MGIoU begins by constructing a set of unique directional Normals \(\mathcal{A}\) . In 2D polygons, these Normals originate from edges, computed as 90-degree rotations of consecutive vertex difference vectors. For ellipses or ellipsoids, the normals align with semi-axes directions. In 3D shapes, these Normals correspond to face normals derived from the surface geometry. Extremely regular shapes (\eg, rectangles in 2D, cuboids in 3D, see \cref{fig:teaser}) typically produce duplicate or parallel Normals; these redundancies are eliminated, leaving only a small number of unique directional Normals that capture essential geometric attributes. For example, a 2D rotated rectangle (\(N=4\)) has just two unique Normals; a 3D cuboid (\(N=8\)) only three. This efficiency reduces computational complexity significantly.
\vspace{-0.5em}
\subsection{MGIoU and MGIoU$^+$ for Structured and Unstructured Shapes}
\vspace{-0.5em}

For structured convex shapes (\eg, rectangles, cuboids), MGIoU maximizes overlap between the predicted shape \(P\) and ground truth \(G\) by projecting their vertices onto \(\mathcal{A}\). The GIoU$^{1D}$ is computed per Normal $a_i \in \mathcal{A}$ using a simplified version of GIoU for 1D (see Appendix), and then averaged across normals \(\text{MGIoU} = \frac{1}{|\mathcal{A}|} \sum_{\mathbf{a} \in \mathcal{A}} \text{GIoU}^{1D}_\mathbf{a}\), finally \(\mathcal{L}_{\text{MGIoU}} = \frac{1 - \text{MGIoU}}{2}\).We have \(\mathcal{L}_{\text{MGIoU}} = 0\) for perfect overlap (\(\text{MGIoU} = 1\)),
\(\mathcal{L}_{\text{MGIoU}} \to 1^-\) as shapes separate (\(\text{MGIoU} \to -1^+\)). As shown in \cref{algo:MGIoU}, the MGIoU is positive when projections overlap and negative when they are disjoint, resulting in \(\text{MGIoU} \in (-1, 1]\), and the final loss \(\mathcal{L}_{\text{MGIoU}} \in [0, 1)\). Given its connection to the Jaccard index and the properties of GIoU, \(\mathcal{L}_{\text{MGIoU}}\) satisfies key properties that make it a robust loss function for shape optimization (See Appendix):
\begin{enumerate}
    \item \textbf{Non-negativity}: \(\mathcal{L}_{\text{MGIoU}}(P, G) \geq 0\)  
    \item \textbf{Identity}: \(\mathcal{L}_{\text{MGIoU}}(P, G) = 0\) if and only if \(P = G\)  
    \item \textbf{Symmetry}: \(\mathcal{L}_{\text{MGIoU}}(P, G) = \mathcal{L}_{\text{MGIoU}}(G, P)\)  
    \item \textbf{Triangle Inequality}: \\
    \(\mathcal{L}_{\text{MGIoU}}(P, R) \leq \mathcal{L}_{\text{MGIoU}}(P, Q) + \mathcal{L}_{\text{MGIoU}}(Q, R)\)  
    \item \textbf{Scale Invariance}: \(\mathcal{L}_{\text{MGIoU}}(sP, sG) = \mathcal{L}_{\text{MGIoU}}(P, G)\) for any scalar \(s > 0\)  
\end{enumerate}
\vspace{-0.5em}
These properties enhance \(\mathcal{L}_{\text{MGIoU}}\) correlation with the IoU metric by effectively measuring shape overlap in projected space, making it a reliable and consistent loss function. 


For unstructured convex shapes (\eg, polygons with varying vertices), we apply GIoU\(^{1D}\) to all shape Normals and introduce a convexity regularizer (\(\mathcal{L}_{\text{convexity}}\)) to ensure geometric consistency. The convexity loss, \(\mathcal{L}_{\text{convexity}}\), is added for unstructured shapes like polygons to prevent unrealistic concavities, which the MGIoU loss alone might allow. It penalizes vertices that fall on the wrong side of any edge’s outward normal, using signed distances to enforce a convex shape, as outlined in Algorithm \ref{algo:MGIoU}. Specifically, for each edge \(i\), we compute the outward normal \(\mathbf{n}_i\) and then calculate the signed distances of all other vertices to this edge. If a vertex lies inside the half-plane defined by the edge (i.e., on the wrong side), it contributes to the penalty. The penalty for each edge is the minimum of the sums of positive and negative signed distances, ensuring that the loss is zero only when all vertices are on the correct side of every edge. This convexity loss is then averaged over all edges to obtain \(\mathcal{L}_{\text{convexity}}\). The signed distance provides a continuous measure of deviation from convexity, guiding the optimization toward a valid shape, which is especially critical for unstructured polygons with variable vertex counts. Combined with the MGIoU loss as \(\mathcal{L}_{\text{MGIoU}^+} = \mathcal{L}_{\text{MGIoU}} + \lambda \mathcal{L}_{\text{convexity}}\), it ensures both accurate overlap and geometric consistency.

\subsection{\textbf{MGIoU$^-$} for Minimizing Shape Overlaps}
\vspace{-0.5em}

\begin{algorithm}[!t]
\caption{\textbf{MGIoU$^-$} - Minimize Shape Overlap}
\label{algo:MGIoU-}
\small{
\SetKwInOut{Input}{input}\SetKwInOut{Output}{output}
\Input{Trajectory boxes \(\mathbf{T} = \{\mathbf{T}^t_i \in \mathbb{R}^{4 \times 2} \mid t=1,\ldots,T; i=1,\ldots,B\}\), masks \(\mathbf{M} = \{m^t_i\}\), scores \(\mathbf{S} = \{s_i\}\)}
\Output{\(\mathcal{L}_{\text{MGIoU$^-$}}\)}
}
Initialize \(\mathcal{L}_{\text{MGIoU$^-$}} = 0\) \\
\For{each \(t = 1\) to \(T\)}{
  Initialize \(\text{loss}^t_i = 0\) for \(i=1,\ldots,B\) \\
  \For{each pair \((i, j)\), \(i \neq j\)}{
    \(\mathcal{A}_{ij}^t = \{\text{normals of } \mathbf{T}^t_i, \mathbf{T}^t_j\}\) \\
    Initialize \(\mathcal{O}_{ij}^t = []\) \\
    \For{each \(\mathbf{a}_k \in \mathcal{A}_{ij}^t\)}{
      \(\text{proj}_{\mathbf{T}^t_i, k} = \mathbf{T}^t_i \cdot \mathbf{a}_k\), \(\text{proj}_{\mathbf{T}^t_j, k} = \mathbf{T}^t_j \cdot \mathbf{a}_k\) \\
      Compute \(\min\) and \(\max\) of \(\text{proj}_{\mathbf{T}^t_i, k}\), \(\text{proj}_{\mathbf{T}^t_j, k}\) \\
      \( |P \cap G| = \max(0, \min(\max_{\mathbf{T}^t_i}, \max_{\mathbf{T}^t_j}) - \max(\min_{\mathbf{T}^t_i}, \min_{\mathbf{T}^t_j})) \) \\
      \( |P| = \max_{\mathbf{T}^t_i} - \min_{\mathbf{T}^t_i} \), \( |G| = \max_{\mathbf{T}^t_j} - \min_{\mathbf{T}^t_j} \) \\
      \( |P \cup G| = |P| + |G| - |P \cap G| \), \( |C| = \max(\max_{\mathbf{T}^t_i}, \max_{\mathbf{T}^t_j}) - \min(\min_{\mathbf{T}^t_i}, \min_{\mathbf{T}^t_j}) \) \\
      \(\text{IoU} = \frac{|P \cap G|}{|P \cup G|} \) \\
      \(\text{GIoU}^{1D}_k = \text{IoU} - \frac{|C| - |P \cup G|}{|C|}\) \\
      Append \(\text{GIoU}^{1D}_k\) to \(\mathcal{O}_{ij}^t\)
    }
    Take the smallest \(\text{GIoU}^{1D}_{\text{min}}\) from \(\mathcal{O}_{ij}^t\) \\
    \(K_{ij}^t = \text{softplus}\left( \text{GIoU}^{1D}_{\text{min}} \right)\) \\
    \(\text{loss}^t_i = \text{loss}^t_i + K_{ij}^t\)
  }
}
\For{each \(i = 1\) to \(B\)}{
  \(L_i = \sum_{t=1}^T m^t_i \cdot \text{loss}^t_i\) \\
  \(\mathcal{L}_{\text{MGIoU$^-$}} = \mathcal{L}_{\text{MGIoU$^-$}} + s_i \cdot L_i\)
}
\Return{\(\mathcal{L}_{\text{MGIoU$^-$}}\)}

\end{algorithm}

Instead of maximizing parametric shape overlaps like MGIoU and MGIoU$^+$, we propose \textbf{MGIoU$^-$} variant minimize them. We chose trajectory prediction task to showcase our loss capability. Here, the goal shifts from maximizing overlap to minimizing it, ensuring predicted trajectories remain separated. Inspired by Separating Axis Theorem, MGIoU$^-$ adapts the MGIoU framework by computing pairwise overlaps between trajectory boxes at each timestep $t$. MGIoU$^-$ naturally extends to temporal scenarios by computing overlap penalties between predicted trajectory boxes at each timestep \( t \) for all pairs of trajectories \( i \) and \( j \). For each timestep \( t \), the algorithm evaluates the smallest 1D GIoU (GIoU\(^{1D}_{\text{min}}\)) across the normals of trajectory pairs, penalizing overlaps using a softplus function. These penalties are summed across all timesteps \( t = 1 \) to \( T \), weighted by ground truth masks \( m^t_i \), ensuring the loss accounts for the entire prediction horizon. For each trajectory \( i \), the loss \( L_i \) integrates temporal penalties over \( T \) timesteps, modulated by masks to handle invalid or variable-length sequences. The final MGIoU$^-$ loss is computed by summing each trajectory’s loss \( L_i \), weighted by prediction scores \( s_i \), prioritizing confident predictions across all timesteps. This temporal formulation ensures that overlaps are minimized throughout the sequence, promoting collision-free trajectories over \( T \) timesteps. The final \(\mathcal{L}_{\text{MGIoU$^-$}}\) is normalized by the classification head’s confidence scores, encouraging the model to either select safer trajectories or adjust their location/orientation regressions to reduce overlaps. Efficiency-wise, MGIoU$^-$ also benefits from parallelizable projections, with a cost of \(O(B^2 \cdot T \cdot |\mathcal{A}|)\) per batch, where \(B\) is the number of trajectories, \(T\) is the number of timesteps, and \(|\mathcal{A}|\) is typically 4 in 2D, as outlined in \cref{algo:MGIoU-}.


In summary, MGIoU and its variants (MGIoU$^+$, MGIoU$^-$) form a unified framework for shape optimization, offering flexibility and efficiency across applications.
\vspace{-0.5em}
\section{Experimental Settings}
\vspace{-0.5em}
\label{sec:settings}

\begin{table*}[!ht]\centering
\resizebox{\linewidth}{!}{%
\begin{tabular}{p{2.2cm}|ccccccccc}\toprule
\textbf{Loss} & \textbf{SUNRGBD} & \textbf{Hypersim} & \textbf{ARKitScenes} & \textbf{Objectron} & \textbf{KITTI} & \textbf{nuScenes} & \textbf{Omni3D$_{\text{Out}}$} & \textbf{Omni3D$_{\text{In}}$} & \textbf{Omni3D} \\\midrule
\textbf{L1 + SO3} & 14.19 & 6.44 & 41.27 & 52.87 & 30.86 & 26.42 & 29.22 & 19.34 & 22.21 \\\cmidrule{1-1}
\textbf{Distangled-L1 + Chamfer} & \raisebox{-1.2ex}{14.64} & \raisebox{-1.2ex}{6.89} & \raisebox{-1.2ex}{40.59} & \raisebox{-1.2ex}{54.94} & \raisebox{-1.2ex}{28.37} & \raisebox{-1.2ex}{27.07} & \raisebox{-1.2ex}{29.93} & \raisebox{-1.2ex}{20.02} & \raisebox{-1.2ex}{22.82} \\\cmidrule{1-1}
\textbf{MGIoU} & \textbf{16.82} & \textbf{8.22} & \textbf{43.76} & \textbf{56.84} & \textbf{31.95} & \textbf{29.66} & \textbf{32.08} & \textbf{21.87} & \textbf{24.86} \\\midrule
\textcolor{darkgreen}{\textbf{Rel. Imp. \%}} & 
{\textcolor{darkgreen}{\textbf{$\uparrow$18.34\%}}} & 
{\textcolor{darkgreen}{\textbf{$\uparrow$5.93\%}}} & 
{\textcolor{darkgreen}{\textbf{$\uparrow$9.03\%}}} & 
{\textcolor{darkgreen}{\textbf{$\uparrow$1.15\%}}} & 
{\textcolor{darkgreen}{\textbf{$\uparrow$15.20\%}}} & 
{\textcolor{darkgreen}{\textbf{$\uparrow$15.06\%}}} & 
{\textcolor{darkgreen}{\textbf{$\uparrow$10.27\%}}} & 
{\textcolor{darkgreen}{\textbf{$\uparrow$7.12\%}}} & 
{\textcolor{darkgreen}{\textbf{$\uparrow$9.01\%}}} \\
\bottomrule
\end{tabular}}
\vspace{-0.5em}
\caption{AP3D Comparison for 3D 6-DoF Object Detection on Omni3D. Baseline model is CubeRCNN~\cite{omni3d}. Improvement values are relative to the Distangled-L1+Chamfer loss~\cite{omni3d}. MGIoU consistently outperform the baselines in both indoor and outdoor settings}
\label{tab:3d_detection}
\vspace{-1em}
\end{table*}

This section briefly describes the datasets, baseline models, and training setups used to evaluate the MGIoU, MGIoU$^+$, and MGIoU$^-$ losses across four tasks: 2D oriented object detection, monocular 3D 6-DoF object regconition, quadrangle object detection, and collision avoidance in trajectory prediction.
\begin{figure}[tp]
    \centering
    \includegraphics[width=0.95\linewidth]{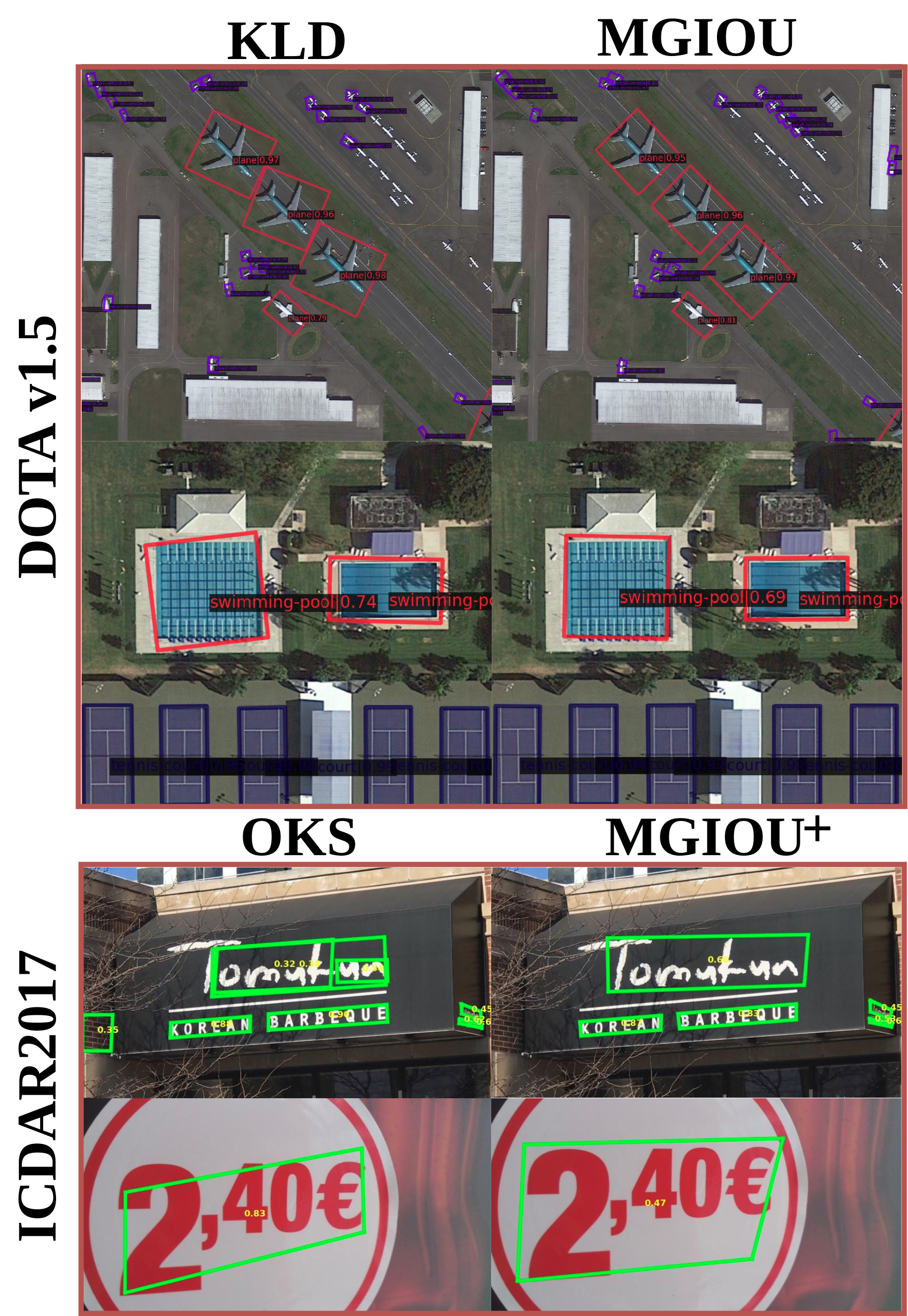}
    \caption{Qualitative visualisation (Test set images) comparing MGIoU vs KLD losses on DOTA dataset, and MGIoU$^+$ vs OKS distance on ICDAR2017 Dataset. MGIoU and MGIoU$^+$ can capture the orientation (2D oriented detection) and vertices better (quadrilaterals detection)}.
    \label{fig:dotaicdar}
    \vspace{-3em}
\end{figure}

\noindent \textbf{2D Oriented Object Detection:} We used the DOTAv1.5~\cite{xia2018dota} dataset (single-scaled split). The baseline model was RetinaNet~\cite{retinanet} (similar baseline to ~\cite{kld,gwd,kfiou}), trained for 12 epochs with SGD, a step LR scheduler (epochs 8 and 11), and a 500-iteration linear warmup.

\noindent \textbf{Monocular 3D 6-DoF Object Recognition:} We select Omni3D~\cite{omni3d} large-scale dataset, which includes SUNRGBD~\cite{sunrgbd}, Hypersim~\cite{hypersim}, ARKitScenes~\cite{arkitscenes}, Objectron~\cite{objectron}, KITTI~\cite{kitti}, and nuScenes~\cite{nuscenes} datasets. This diversity ensures a robust test bed across both indoor and outdoor 3D environments. The baseline model is CubeRCNN~\cite{omni3d} proposed by dataset authors, trained for 5,568,000 iterations on a 4090 GPU. The training adhered to the configurations detailed in the Omni3D~\cite{omni3d} paper, employing the SGD optimizer alongside a step learning rate scheduler.

\begin{figure*}[tp]
    \centering
    \includegraphics[width=\linewidth]{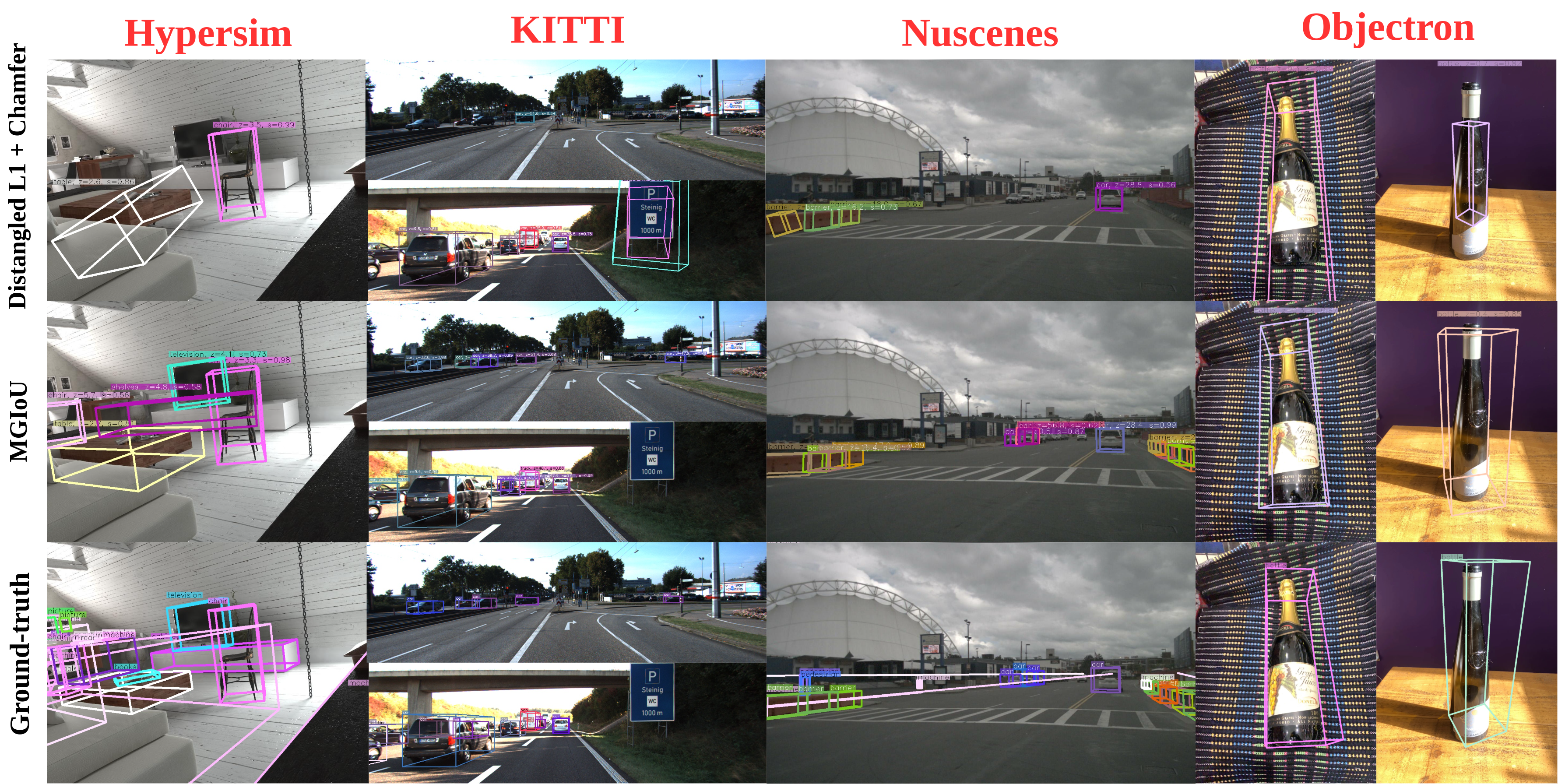}
    \vspace{-0.5em}
    \caption{Qualitative visualisation on Omni3D dataset.}
    \vspace{-1.5em}
    \label{fig:omni3dcomparison}
\end{figure*}

\noindent \textbf{Quadrangle Object Detection:} The quadrangle object detection task, we choose ICDAR2017 \cite{icdar} dataset. The ICDAR2017 competition includes the MLT (Multilingual Text Detection) task, covering text in nine languages and various orientations. This dataset is particularly suitable for quadrangle object detection due to the irregular shapes of text regions, often represented as quadrilaterals. The baseline model selected was YOLO-NAS~\cite{supergradients} due to good performance and ease of implementation, we train for 40 epochs with AdamW and a cosine LR scheduler.

\noindent \textbf{Collision Avoidance in Trajectory Prediction:} The Waymo~\cite{waymo} motion prediction dataset was utilized together with MTR~\cite{mtr} baseline. We trained the model on a subset of 20\% of the training set and report results on validation set, values in \cref{tab:waymo_combined} is for the most confident trajectory per object (not the average of top 6 trajectories). The original model architecture was retained without modifications, so that we can isolate MGIoU$^-$ and evaluate its effectiveness in reducing collisions and enhancing prediction accuracy.

\vspace{-0.5em}
\section{Comparisons with losses and discussions}
\vspace{-0.5em}
\label{sec:experiments}

To evaluate the effectiveness of the Marginalized Generalized IoU (MGIoU) loss, we conducted experiments across the four tasks detailed in Section \ref{sec:settings}. Below, we present the results for each task, including performance metrics and a discussion of their implications, highlighting MGIoU's superior performance compared to other loss functions.
\vspace{-0.5em}
\subsection{2D Oriented Object Detection}
\vspace{-0.5em}
For 2D oriented object detection on the DOTAv1.5 dataset, we compared MGIoU with L1, KFIoU\cite{kfiou}, GWD\cite{gwd}, and KLD\cite{kld} losses using RetinaNet\cite{retinanet} as the baseline model, implemented in mmrotate library \cite{mmrotate}. Performance was measured using mean Average Precision (mAP) and latency (ms) on a single 4090 GPU. Table \ref{tab:2d_detection} shows that MGIoU achieved the highest mAP of 0.554 with a loss computation latency of 0.45 ms, outperforming L1, KFIoU, GWD, and KLD. Its loss computation latency is significantly lower than KFIoU (51.1x slower), GWD (16.9x slower), and KLD (17.3x slower), and only slightly higher than L1. This balance of high accuracy and low latency makes MGIoU ideal for faster and accurate training.

\begin{table}[ht]
\centering
\resizebox{0.8\linewidth}{!}{%
\small
\begin{tabular}{l|cccc}\toprule
\textbf{Loss} & \textbf{mAP} & \textbf{Latency (ms)} & \textbf{Rel. Latency} \\\midrule
\textbf{L1} & 0.522 & \textbf{0.03} & \textcolor{darkgreen}{$\times$0.10} \\
\textbf{KFIoU\cite{kfiou}} & 0.546 & 23 & \textcolor{red}{$\times$51.1} \\
\textbf{GWD\cite{gwd}} & 0.547 & 7.6 & \textcolor{red}{$\times$16.9} \\
\textbf{KLD\cite{kld}} & 0.55 & 7.8 & \textcolor{red}{$\times$17.3} \\\midrule
\textbf{MGIoU (ours)} & \textbf{0.554} & 0.45 & $\times$1.00 \\
\bottomrule
\end{tabular}%
}
\vspace{-0.5em}
\caption{Comparison of Losses on DOTAv1.5 dataset}\label{tab:2d_detection}
\vspace{-1em}
\end{table}
\vspace{-0.5em}
\subsection{Monocular 3D 6-DoF Object Recognition}
\vspace{-0.5em}
For 3D 6-DoF object recognition on the Omni3D dataset, we evaluated MGIoU using CubeRCNN~\cite{omni3d} as the baseline model. Performance was assessed with AP3D across various indoor and outdoor scenes. Table \ref{tab:3d_detection} demonstrates that MGIoU consistently improved AP3D across all datasets compared to Distangled-L1 + Chamfer, with gains from 1.15\% (Objectron) to 18.34\% (SUNRGBD). The overall Omni3D improvement was 2.04 (9\% relative improvement), showing MGIoU’s robustness in both indoor (Omni3D$_{\text{In}}$) and outdoor (Omni3D$_{\text{Out}}$) settings. Other losses typically involve multiple components—such as separate terms for location, dimension, and rotation—which require hyperparameter tuning to balance their contributions. For example, in the baseline Distangled-L1 + Chamfer loss, the L1 loss handles location and dimension, while Chamfer distance addresses rotation discrepancies. This multi-component setup necessitates careful tuning to optimize the trade-off between these aspects, as illustrated in \cref{fig:omni3dcomparison}. In contrast, MGIoU unifies location, dimension, and rotation into a single loss function, optimizing them holistically without the need for additional hyperparameters. There are cases where the baseline predicts a nearly perfect location, but slight errors in dimension or rotation lead to significant drops in 3D IoU, and vice versa. As shown in \cref{fig:omni3dcomparison}, the baseline’s predicted boxes often show misalignments, whereas MGIoU’s boxes better match the ground-truth boxes, mitigating these issues.

\begin{figure*}[tp]
    \centering
    \includegraphics[width=0.9\linewidth]{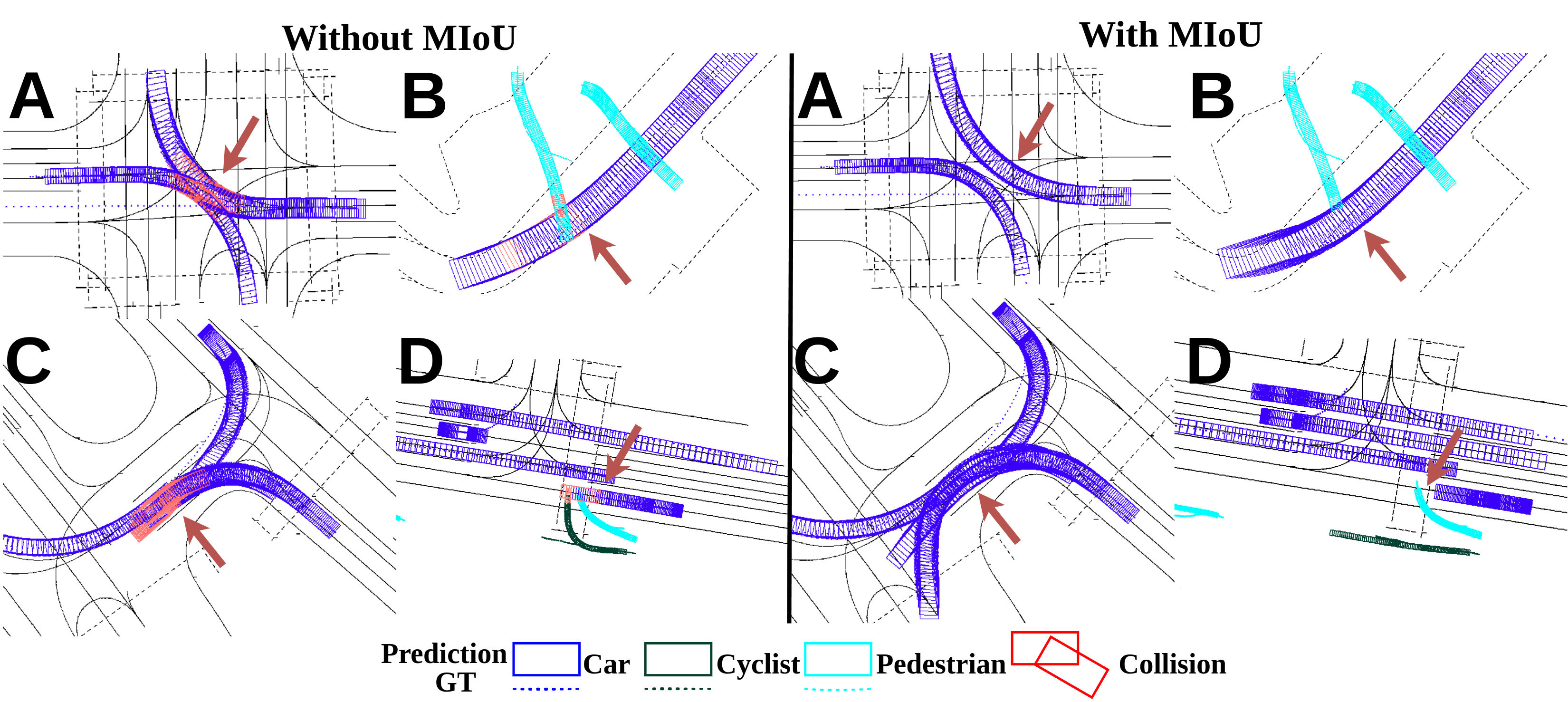}
    \caption{Qualitative visualisation on Waymo dataset, we visualize the predicted future bounding boxes of road agents in the next 8 seconds (80 timesteps), with and without MGIoU$^-$ during training. With MGIoU$^-$ incorporated in the training stage, model now have a better understanding of the physical world and can make safer interactions between road agents. \textbf{A)} 2 cars avoid collision at an intersection, \textbf{B)} Pedestrian stops and wait, \textbf{C)} Smooth interaction among three vehicles without collisions, \textbf{D)} Car appropriately yields to a pedestrian}.
    \label{fig:waymo_comparison}
    \vspace{-1.5em}
\end{figure*}

\begin{table*}[tp]
\centering
\resizebox{0.75\linewidth}{!}{%
\begin{tabular}{lcccccccc}
\toprule
 & \multicolumn{4}{c}{\textbf{Without MGIoU$^-$}} & \multicolumn{4}{c}{\textbf{With MGIoU$^-$}}\\
\cmidrule(lr){2-5} \cmidrule(lr){6-9}
\textbf{Category} & \textbf{mAP} & \textbf{minADE} & \textbf{minFDE} & \textbf{MissRate} & \textbf{mAP} & \textbf{minADE} & \textbf{minFDE} & \textbf{MissRate} \\
\midrule
\textbf{VEHICLE}   & 0.3026 & 1.8511 & 4.6943 & 0.5041 & \textbf{0.3087} & 1.8232 & 4.6481 & 0.4975 \\
\textbf{PEDESTRIAN}& 0.3046 & 0.6557 & 1.5382 & 0.3243 & \textbf{0.3077} & 0.6603 & 1.5551 & 0.3304 \\
\textbf{CYCLIST}   & 0.2397 & 1.5937 & 3.7869 & 0.5116 & \textbf{0.2720} & 1.5511 & 3.7105 & 0.4941 \\
\midrule
\textbf{Avg}       & 0.2823 & 1.3668 & 3.3398 & 0.4466 & \textbf{0.2961} & 1.3449 & 3.3046 & 0.4407 \\
\midrule
\textbf{\# collision} & \multicolumn{4}{c}{7493} & \multicolumn{4}{c}{\textbf{6443 (\(\downarrow\) 14\%)} } \\
\bottomrule
\end{tabular}
}
\vspace{-0.5em}
\caption{Trajectory Prediction on Waymo Dataset (Collisions out of 192,172 predicted trajectories, exclude collisions between pedestrians)}
\label{tab:waymo_combined}
\vspace{-1.5em}
\end{table*}

\vspace{-0.5em}
\subsection{Quadrangle Object Detection}
\vspace{-0.5em}
For quadrangle object detection on the ICDAR2017 dataset, we used YOLO-NAS~\cite{supergradients} as the baseline model. Performance was evaluated using Average Precision (AP) and Average Recall (AR) at an IoU threshold of 0.5. Table \ref{tab:quad_detection} reveals that MGIoU outperformed L1 and OKS distance losses in both AP and AR, which demonstrate MGIoU’s effectiveness in handling irregularly shaped objects like quadrangles in scene text detection.

\begin{table}[h]
\centering
\resizebox{0.6\linewidth}{!}{%
\begin{tabular}{lcc}
\toprule
\textbf{Loss} & \textbf{AP} & \textbf{AR} \\
\midrule
\textbf{QRN \cite{qrn}}            & 48.60       & 42.67       \\
\textbf{OKS distance \cite{oks}}  & 49.10       & 43.39       \\
\textbf{MGIoU$^+$ (Ours)}         & \textbf{49.84}       & \textbf{44.91}       \\
\bottomrule
\end{tabular}
}
\caption{Comparison of Losses on ICDAR2017 dataset.}
\vspace{-1em}
\label{tab:quad_detection}
\end{table}

\vspace{-0.5em}
\subsection{Collision Avoidance in Trajectory Prediction}
\vspace{-0.5em}
For collision avoidance in trajectory prediction on the Waymo dataset, we incorporated MGIoU$^-$ into the model without altering its architecture. We evaluated its impact on collision reduction and prediction metrics (mAP, minADE, minFDE, and MissRate). Table \ref{tab:waymo_combined} indicates that MGIoU$^-$ improved average mAP from 0.2823 to 0.2961 with the improvement across all classes. Additionally, collisions dropped by 14\% (from 7,493 to 6,443), showcasing MGIoU$^-$’s ability to enhance prediction quality and safety in autonomous driving. With MGIoU$^-$, the model gains a heightened awareness of the physical world and safety constraints, enabling safer interactions between road agents. As shown in \cref{fig:waymo_comparison}, the model learns to give way to pedestrians, safely merge onto highways, wait for other vehicles at intersections, and follow behind other cars more effectively. For example, panels with red collision areas in the baseline (without MGIoU$^-$) are replaced by collision-free predictions in the MGIoU$^-$ version, reflecting a 14\% reduction in collisions. These improvements stem from MGIoU$^-$’s ability to penalize overlaps between predicted trajectories, encouraging the model to generate paths that avoid collisions while maintaining high prediction accuracy.
\vspace{-1em}
\section{Conclusion}
\vspace{-0.5em}
This paper addresses the critical challenge of optimizing parametric shapes in computer vision by introducing Marginalized Generalized IoU (MGIoU), a novel loss function that unifies and enhances shape optimization across diverse applications. Extensive experiments demonstrate that MGIoU and its variants consistently outperform strong objective functions in multiple domains, offering a robust and versatile solution for unifying convex parametric shapes optimization.
{
\small
    \bibliographystyle{ieeenat_fullname}
    \bibliography{main}
}

\clearpage
\setcounter{page}{1}
\maketitlesupplementary

\section{Simplification of One-Dimensional GIoU Formula}

We aim to prove that the simplified formula for the one-dimensional Generalized Intersection over Union (GIoU) metric for projections onto axis \( A_i \) is equivalent to the standard GIoU definition. The \textbf{simplified formula that is used throughout our paper}, instead of the original GIOU formula, is:

\begin{equation}
    \text{GIoU}_{A_i}^{1D} = \frac{\min(\max_{P,A_i}, \max_{G,A_i}) - \max(\min_{P,A_i}, \min_{G,A_i})}{\max(\max_{P,A_i}, \max_{G,A_i}) - \min(\min_{P,A_i}, \min_{G,A_i})}
    \label{eq:simplified_giou}
\end{equation}

\noindent We will demonstrate that this matches the standard GIoU definition:

\[
\text{GIoU}_{A_i}^{1D} = \frac{|P_{A_i} \cap G_{A_i}|}{|P_{A_i} \cup G_{A_i}|} - \frac{|C_{A_i} \setminus (P_{A_i} \cup G_{A_i})|}{|C_{A_i}|}
\]

\noindent where:\\
- \( P_{A_i} = [\min_{P,A_i}, \max_{P,A_i}] \) is the projection of the predicted set onto axis \( A_i \),\\
- \( G_{A_i} = [\min_{G,A_i}, \max_{G,A_i}] \) is the projection of the ground truth set onto axis \( A_i \),\\
- \( C_{A_i} \) is the smallest interval enclosing both \( P_{A_i} \) and \( G_{A_i} \),\\
- \( | \cdot | \) denotes the length of an interval.

\noindent\textbf{Step 1}: Define the Intervals and Key Quantities

Let:\\
- \( a = \min_{P,A_i} \) (left endpoint of \( P_{A_i} \)),\\
- \( b = \max_{P,A_i} \) (right endpoint of \( P_{A_i} \)),\\
- \( c = \min_{G,A_i} \) (left endpoint of \( G_{A_i} \)),\\
- \( d = \max_{G,A_i} \) (right endpoint of \( G_{A_i} \)).

\noindent We assume \( a \leq b \) and \( c \leq d \), which is true for valid intervals. The smallest enclosing interval \( C_{A_i} \) is:

\[
C_{A_i} = [\min(a, c), \max(b, d)]
\]

\noindent Its length is:

\[
|C_{A_i}| = \max(b, d) - \min(a, c)
\]

\noindent\textbf{Step 2}: Compute the Intersection

\noindent The intersection \( P_{A_i} \cap G_{A_i} \) is the overlap between \( [a, b] \) and \( [c, d] \). Its length is:

\[
|P_{A_i} \cap G_{A_i}| = \max\left(0, \min(b, d) - \max(a, c)\right)
\]

- If the intervals overlap (\( \max(a, c) < \min(b, d) \)), the intersection length is positive. \\
- If they are disjoint (e.g., \( b < c \) or \( d < a \)), the intersection length is zero.

\noindent\textbf{Step 3}: Compute the Union

\noindent The union \( P_{A_i} \cup G_{A_i} \) spans all points covered by \( [a, b] \) or \( [c, d] \). The length of the union in one dimension is:

\[
|P_{A_i} \cup G_{A_i}| = \max(b, d) - \min(a, c)
\]

\noindent This expression holds because:\\
- If the intervals overlap, \( P_{A_i} \cup G_{A_i} = [\min(a, c), \max(b, d)] \).\\
- If disjoint, this still represents the span from the smallest to the largest endpoint.

\noindent\textbf{Step 4}: Compute the Penalty Term

\noindent The penalty term \( |C_{A_i} \setminus (P_{A_i} \cup G_{A_i})| \) is the length of the smallest enclosing interval not covered by the union:
- If \( P_{A_i} \) and \( G_{A_i} \) overlap or touch, \( P_{A_i} \cup G_{A_i} \) fills \( C_{A_i} \), so:

\[
|C_{A_i} \setminus (P_{A_i} \cup G_{A_i})| = 0
\]

- If disjoint (e.g., \( b < c \)), then \( C_{A_i} = [\min(a, c), \max(b, d)] \), and the gap is:

\[
|C_{A_i} \setminus (P_{A_i} \cup G_{A_i})| = c - b \quad (\text{if } b < c)
\]

\noindent Generally:

\[
|C_{A_i} \setminus (P_{A_i} \cup G_{A_i})| = \max\left(0, \max(a, c) - \min(b, d)\right)
\]

\noindent\textbf{Step 5}: Verify the Simplified Formula

\noindent We verify the simplified formula:

\[
\text{GIoU}_{A_i}^{1D} = \frac{\min(b, d) - \max(a, c)}{\max(b, d) - \min(a, c)}
\]

\noindent\textbf{Case 1}: Overlapping Intervals

\noindent Consider \( a < c < b < d \) (overlapping case):\\
- Intersection: \( |P_{A_i} \cap G_{A_i}| = \min(b, d) - \max(a, c) = b - c \),\\
- Union: \( |P_{A_i} \cup G_{A_i}| = \max(b, d) - \min(a, c) = d - a \),\\
- Penalty: \( |C_{A_i} \setminus (P_{A_i} \cup G_{A_i})| = 0 \) (union fills \( C_{A_i} \)),\\
- \( |C_{A_i}| = d - a \).

\noindent Standard GIoU:

\[
\text{GIoU} = \frac{b - c}{d - a} - \frac{0}{d - a} = \frac{b - c}{d - a}
\]

\noindent Simplified formula:

\[
\frac{\min(b, d) - \max(a, c)}{\max(b, d) - \min(a, c)} = \frac{b - c}{d - a}
\]

\noindent The formulas match.

\noindent\textbf{Case 2}: Disjoint Intervals

\noindent Consider \( a < b < c < d \) (disjoint case):\\
- Intersection: \( |P_{A_i} \cap G_{A_i}| = \max\left(0, \min(b, d) - \max(a, c)\right) = \max\left(0, b - c\right) = 0 \) (since \( b < c \)),\\
- Union: \( |P_{A_i} \cup G_{A_i}| = \max(b, d) - \min(a, c) = d - a \),\\
- \( |C_{A_i}| = d - a \),\\
- Penalty: \( |C_{A_i} \setminus (P_{A_i} \cup G_{A_i})| = c - b \).

\noindent Standard GIoU:

\[
\text{IoU} = \frac{0}{(b - a) + (d - c)} = 0,
\]
\[
 \quad \text{GIoU} = 0 - \frac{c - b}{d - a} = -\frac{c - b}{d - a}
\]
\noindent Simplified formula:

\[
\frac{\min(b, d) - \max(a, c)}{\max(b, d) - \min(a, c)} = \frac{b - c}{d - a} = -\frac{c - b}{d - a}
\]

\noindent The formulas match.

\noindent\textbf{Step 6}: Conclusion

\noindent The simplified formula \cref{eq:simplified_giou} correctly computes the one-dimensional GIoU:\\
- For overlapping intervals, it equals the IoU (penalty term is zero).\\
- For disjoint intervals, it yields a negative value, \( -\frac{\text{gap}}{|C_{A_i}|} \), consistent with GIoU when IoU is zero.

\noindent This completes the proof, confirming that the simplified formula is a valid representation of the GIoU metric for one-dimensional projections.

\section{Proof of Properties for MGIoU}
\label{supp:proof_metric_property}

To assess the robustness of the MGIoU as a similarity measure for shape optimization, we define the loss function \( \mathcal{L}_{\text{MGIoU}}(P, G) = \frac{1 - \text{MGIoU}(P, G)}{2} \) and evaluate its properties. Specifically, we analyze the following properties for \( \mathcal{L}_{\text{MGIoU}} \) over structured convex shapes \( P \) and \( G \) with a shared parametric domain:
\begin{enumerate}
    \item \textbf{Non-negativity}: \( \mathcal{L}_{\text{MGIoU}}(P, G) \geq 0 \),
    \item \textbf{Identity}: \( \mathcal{L}_{\text{MGIoU}}(P, G) = 0 \) if and only if \( P = G \),
    \item \textbf{Symmetry}: \( \mathcal{L}_{\text{MGIoU}}(P, G) = \mathcal{L}_{\text{MGIoU}}(G, P) \),
    \item \textbf{Triangle Inequality}: \( \mathcal{L}_{\text{MGIoU}}(P, R) \leq \mathcal{L}_{\text{MGIoU}}(P, Q) + \mathcal{L}_{\text{MGIoU}}(Q, R) \),
    \item \textbf{Scale-Invariance}: \( \mathcal{L}_{\text{MGIoU}}(sP, sG) = \mathcal{L}_{\text{MGIoU}}(P, G) \) for any scalar \( s > 0 \).
\end{enumerate}

\subsection*{Definition of MGIoU}
MGIoU is defined as the average of one-dimensional Generalized Intersection over Union (GIoU) values computed across a set of projection directions:
\[
\text{MGIoU}(P, G) = \frac{1}{|\mathcal{A}|} \sum_{i \in \mathcal{A}} \text{GIoU}^{1D}_i(P_i, G_i)
\]
where:\\
- \( \mathcal{A} \) is the set of unique normal directions derived from the surfaces of \( P \) and \( G \),\\
- \( P_i \) and \( G_i \) are the projections of shapes \( P \) and \( G \) onto the \( i \)-th normal direction,\\
- \( \text{GIoU}^{1D}_i(P_i, G_i) \) is the one-dimensional GIoU for the projected intervals.

The associated loss function is:
\begin{align*}
\mathcal{L}_{\text{MGIoU}}(P, G) &= \frac{1 - \text{MGIoU}(P, G)}{2} \\
&= \frac{1}{2} - \frac{1}{2|\mathcal{A}|} \sum_{i \in \mathcal{A}} \text{GIoU}^{1D}_i(P_i, G_i)
\end{align*}

\begin{lemma}[Symmetry of MGIoU]
For any two structured convex shapes \( P \) and \( G \),
\[
\text{MGIoU}(P, G) = \text{MGIoU}(G, P)
\]
\end{lemma}
\begin{proof}
The set \( \mathcal{A} \) comprises unique normals from both \( P \) and \( G \), making it symmetric with respect to the pair \( (P, G) \). For each direction \( i \in \mathcal{A} \), the one-dimensional GIoU, \( \text{GIoU}^{1D}_i(P_i, G_i) \), is symmetric because intersection and union operations are commutative in one dimension. Thus:
\begin{align*}
\text{MGIoU}(P, G) &= \frac{1}{|\mathcal{A}|} \sum_{i \in \mathcal{A}} \text{GIoU}^{1D}_i(P_i, G_i) \\
&= \frac{1}{|\mathcal{A}|} \sum_{i \in \mathcal{A}} \text{GIoU}^{1D}_i(G_i, P_i) \\
&= \text{MGIoU}(G, P)
\end{align*}
Consequently, the loss function satisfies:
\[
\mathcal{L}_{\text{MGIoU}}(P, G) = \mathcal{L}_{\text{MGIoU}}(G, P)
\]
\end{proof}

\begin{lemma}[Identity Property of MGIoU]
For structured convex shapes \( P \) and \( G \) with the same parametric domain,
\[
\text{MGIoU}(P, G) = 1 \quad \text{if and only if} \quad P = G
\]
\end{lemma}
\begin{proof}
- \textit{Forward direction}: If \( P = G \), then for all \( i \in \mathcal{A} \), the projections satisfy \( P_i = G_i \). Since \( \text{GIoU}^{1D}_i(P_i, P_i) = 1 \) (the intersection equals the union), we have:
\[
\text{MGIoU}(P, P) = \frac{1}{|\mathcal{A}|} \sum_{i \in \mathcal{A}} 1 = 1
\]
- \textit{Converse direction}: If \( \text{MGIoU}(P, G) = 1 \), then:
\[
\frac{1}{|\mathcal{A}|} \sum_{i \in \mathcal{A}} \text{GIoU}^{1D}_i(P_i, G_i) = 1
\]
Since \( \text{GIoU}^{1D}_i \leq 1 \) and the average is 1, it follows that \( \text{GIoU}^{1D}_i(P_i, G_i) = 1 \) for all \( i \in \mathcal{A} \). This implies \( P_i = G_i \) for all projection directions. Given that \( \mathcal{A} \) includes all unique normals from \( P \) and \( G \), and the shapes are structured and convex with a shared parametric domain, \( P = G \) in the original space.

Thus, \( \mathcal{L}_{\text{MGIoU}}(P, G) = 0 \) if and only if \( P = G \).
\end{proof}

\begin{lemma}[Scale-Invariance of MGIoU]
For any scalar \( s > 0 \),
\[
\mathcal{L}_{\text{MGIoU}}(sP, sG) = \mathcal{L}_{\text{MGIoU}}(P, G)
\]
\end{lemma}
\begin{proof}
When both shapes \( P \) and \( G \) are scaled by the same factor \( s > 0 \), their projections onto each normal direction \( i \in \mathcal{A} \) scale by \( s \). The one-dimensional GIoU \( \text{GIoU}^{1D}_i(sP_i, sG_i) \) remains unchanged because the intersection, union, and smallest enclosing interval scale proportionally:
\begin{align*}
\text{IoU}(sP_i, sG_i) &= \frac{|sP_i \cap sG_i|}{|sP_i \cup sG_i|} \\
&= \frac{s |P_i \cap G_i|}{s |P_i \cup G_i|} \\
&= \text{IoU}(P_i, G_i)
\end{align*}
\begin{align*}
\frac{|C_{sP_i, sG_i} \setminus (sP_i \cup sG_i)|}{|C_{sP_i, sG_i}|} &= \frac{s |C_{P_i, G_i} \setminus (P_i \cup G_i)|}{s |C_{P_i, G_i}|} \\
&= \frac{|C_{P_i, G_i} \setminus (P_i \cup G_i)|}{|C_{P_i, G_i}|}
\end{align*}
Thus:
\begin{align*}
\text{GIoU}^{1D}_i(sP_i, sG_i) &= \text{IoU}(sP_i, sG_i) - \frac{|C_{sP_i, sG_i} \setminus (sP_i \cup sG_i)|}{|C_{sP_i, sG_i}|} \\
&= \text{IoU}(P_i, G_i) - \frac{|C_{P_i, G_i} \setminus (P_i \cup G_i)|}{|C_{P_i, G_i}|} \\
&= \text{GIoU}^{1D}_i(P_i, G_i)
\end{align*}
Since this holds for all \( i \in \mathcal{A} \), the MGIoU is unchanged:
\begin{align*}
\text{MGIoU}(sP, sG) &= \frac{1}{|\mathcal{A}|} \sum_{i \in \mathcal{A}} \text{GIoU}^{1D}_i(sP_i, sG_i) \\
&= \frac{1}{|\mathcal{A}|} \sum_{i \in \mathcal{A}} \text{GIoU}^{1D}_i(P_i, G_i) \\
&= \text{MGIoU}(P, G)
\end{align*}
Therefore, the loss function is scale-invariant:
\begin{align*}
\mathcal{L}_{\text{MGIoU}}(sP, sG) &= \frac{1 - \text{MGIoU}(sP, sG)}{2} \\
&= \frac{1 - \text{MGIoU}(P, G)}{2} \\
&= \mathcal{L}_{\text{MGIoU}}(P, G)
\end{align*}
\end{proof}

\begin{proposition}[Properties of \( \mathcal{L}_{\text{MGIoU}} \)]
The loss function \( \mathcal{L}_{\text{MGIoU}}(P, G) = \frac{1 - \text{MGIoU}(P, G)}{2} \) satisfies:
\begin{enumerate}
    \item \textbf{Non-negativity}: \( \mathcal{L}_{\text{MGIoU}}(P, G) \geq 0 \),
    \item \textbf{Identity}: \( \mathcal{L}_{\text{MGIoU}}(P, G) = 0 \) if and only if \( P = G \),
    \item \textbf{Symmetry}: \( \mathcal{L}_{\text{MGIoU}}(P, G) = \mathcal{L}_{\text{MGIoU}}(G, P) \),
    \item \textbf{Triangle Inequality}: \( \mathcal{L}_{\text{MGIoU}}(P, R) \leq \mathcal{L}_{\text{MGIoU}}(P, Q) + \mathcal{L}_{\text{MGIoU}}(Q, R) \),
    \item \textbf{Scale-Invariance}: \( \mathcal{L}_{\text{MGIoU}}(sP, sG) = \mathcal{L}_{\text{MGIoU}}(P, G) \) for any scalar \( s > 0 \).
\end{enumerate}
\end{proposition}
\begin{proof}
- \textbf{Non-negativity}: Since \( \text{MGIoU}(P, G) \leq 1 \), we have:
\[
\mathcal{L}_{\text{MGIoU}}(P, G) = \frac{1 - \text{MGIoU}(P, G)}{2} \geq 0
\]
- \textbf{Identity}: From Lemma 2, \( \mathcal{L}_{\text{MGIoU}}(P, G) = 0 \) if and only if \( \text{MGIoU}(P, G) = 1 \), which holds if and only if \( P = G \). \\
- \textbf{Symmetry}: From Lemma 1, \( \mathcal{L}_{\text{MGIoU}}(P, G) = \mathcal{L}_{\text{MGIoU}}(G, P) \).\\
- \textbf{Triangle Inequality}: To prove \( \mathcal{L}_{\text{MGIoU}}(P, R) \leq \mathcal{L}_{\text{MGIoU}}(P, Q) + \mathcal{L}_{\text{MGIoU}}(Q, R) \), assume a fixed set of projection directions \( \mathcal{A} \) for all pairs of shapes under consideration. This assumption simplifies the comparison across different shape pairs by ensuring a consistent set of directions.

Define the one-dimensional distance for each direction \( i \in \mathcal{A} \) as:
\[
d_i(P_i, G_i) = 1 - \text{GIoU}^{1D}_i(P_i, G_i)
\]
Then, the loss function can be expressed as:
\begin{align*}
\mathcal{L}_{\text{MGIoU}}(P, G) &= \frac{1}{2} \left(1 - \text{MGIoU}(P, G)\right) \\
&= \frac{1}{2} \cdot \frac{1}{|\mathcal{A}|} \sum_{i \in \mathcal{A}} d_i(P_i, G_i)
\end{align*}
For each direction \( i \in \mathcal{A} \), the one-dimensional distance \( d_i \) satisfies the triangle inequality:
\[
d_i(P_i, R_i) \leq d_i(P_i, Q_i) + d_i(Q_i, R_i)
\]
This holds because \( \text{GIoU}^{1D}_i \) behaves as a metric-like function in one dimension, where projections reduce to intervals, and the GIoU distance \( 1 - \text{GIoU}^{1D}_i \) inherits the triangle inequality from standard GIoU interval metrics \cite{giou}.

Summing this inequality over all directions \( i \in \mathcal{A} \):
\[
\sum_{i \in \mathcal{A}} d_i(P_i, R_i) \leq \sum_{i \in \mathcal{A}} d_i(P_i, Q_i) + \sum_{i \in \mathcal{A}} d_i(Q_i, R_i)
\]
Multiply both sides by the positive constant \( \frac{1}{2} \cdot \frac{1}{|\mathcal{A}|} \):
\begin{align*}
\frac{1}{2} \cdot \frac{1}{|\mathcal{A}|} \sum_{i \in \mathcal{A}} d_i(P_i, R_i) &\leq \frac{1}{2} \cdot \frac{1}{|\mathcal{A}|} \sum_{i \in \mathcal{A}} d_i(P_i, Q_i) \\
+ \frac{1}{2} \cdot \frac{1}{|\mathcal{A}|} \sum_{i \in \mathcal{A}} d_i(Q_i, R_i)
\end{align*}
This directly corresponds to:
\[
\mathcal{L}_{\text{MGIoU}}(P, R) \leq \mathcal{L}_{\text{MGIoU}}(P, Q) + \mathcal{L}_{\text{MGIoU}}(Q, R)
\]
- \textbf{Scale-Invariance}: From Lemma 3, \( \mathcal{L}_{\text{MGIoU}}(sP, sG) = \mathcal{L}_{\text{MGIoU}}(P, G) \).
\end{proof}
\end{document}